%% file: arxiv.tex
\input{includes}

\IEEEpeerreviewmaketitle

\begin{document}

\title{AURA: Asymptotically Optimal Uncertainty-Robust Replanning Algorithm for Kinodynamic Systems}

\author{Seyedali~Golestaneh,
        Zhuoyun~Zhong,
        Donghyung~Lee,
        and~Constantinos~Chamzas
  \thanks{%
   Manuscript received: May, 26, 2026; Revised: August 25, 2026; Accepted: September 28, 2026. This paper was recommended for publication by Editor Olivier Stasse and upon evaluation of the Associate Editor and reviewers’ comments.
   }
   \thanks{
   This work was supported in part by Amazon WPI RBE Research Award 2025, NSF CRII Grant No. 2451108, and WPI funds. All authors are affiliated with the Department of Robotics Engineering, Worcester Polytechnic Institute (WPI), Worcester, MA 01609, USA {\tt\footnotesize \{sgolestaneh, zzhong3, dlee5, cchamzas\} @ wpi.edu}.
  }
}
\markboth{IEEE ROBOTICS AND AUTOMATION LETTERS. REPRINT VERSION. ACCEPTED OCTOBER, 2026}%
{Shell \MakeLowercase{\textit{et al.}}: A Sample Article Using IEEEtran.cls for IEEE Journals}

\maketitle

\definecolor{changes}{RGB}{0, 0, 0}
\newcommand{\delete}[1]{}

\input{sections/0.Abstract}
\input{sections/1.Introduction}
\input{sections/2.RelatedWork}
\input{sections/3.ProblemStatement}
\input{sections/4.Methodology}
\input{sections/5.Theoretical}
\input{sections/6.Experiments}
\input{sections/7.Conclusion}

\def\url#1{}
\bibliographystyle{IEEEtran_ShortURL}
\bibliography{sections/reference}

\newpage

%

\end{document}

%% file: includes.tex
\documentclass[journal]{IEEEtran}

\usepackage{soul}
\soulregister\method0
\soulregister\cite7
\usepackage{amsthm}
\usepackage{cancel}
\usepackage{mathtools}
\usepackage[normalem]{ulem}
\usepackage{amsmath,amssymb}

\usepackage[pdftex]{graphicx}
\usepackage{svg}
\svgsetup{inkscapelatex=true}
\svgsetup{pretex=\scriptsize}

\usepackage{todonotes}
\usepackage[
activate   = {true},
protrusion = true,
expansion  = true,
kerning    = true,
spacing    = true,
tracking   = false,
auto       = true,
selected   = true,
factor     = 1000,
stretch    = 35,
shrink     = 35,
]{microtype}

\usepackage{xcolor}
\usepackage[dvipsnames]{xcolor}
\definecolor{wpired}{RGB}{34, 153, 84}
\definecolor{rred}{RGB}{169, 50, 38}
\definecolor{ggreen}{RGB}{34, 153, 84}
\definecolor{bblue}{RGB}{36, 113, 163}
\definecolor{ppurple}{RGB}{125, 60, 152}
\definecolor{yyellow}{RGB}{214, 137, 16}
\definecolor{ggrey}{RGB}{112, 123, 124}
\definecolor{wwpi}{RGB}{172, 43, 55}

\usepackage[pdfa,
  colorlinks,
  bookmarks=true,bookmarksopen,bookmarksnumbered,
  allcolors=wwpi]{hyperref}

\usepackage[linesnumbered,ruled,vlined]{algorithm2e}

\usepackage[nameinlink,noabbrev]{cleveref}

\Crefname{algocf}{Algorithm}{Algorithms}

\usepackage{cite}
\usepackage[caption=false,font=normalsize,labelfont=sf,textfont=sf]{subfig}
\usepackage{dblfloatfix}

\usepackage{xspace}
\newcommand{\method}{\textsc{Aura}\xspace}

\newcommand{\R}[1]{{\ensuremath{\mathbb{R}^{#1}}}\xspace}

\usepackage{todonotes}

\SetKwInput{KwIn}{Input}
\SetKwFunction{Plan}{Plan}
\SetKwFunction{ContinuePlan}{ContinuePlan}
\SetKwFunction{Prop}{Propagate}
\SetKwFunction{OptBatch}{OptimizeBatch}
\SetKwFunction{Near}{Nearest}
\SetKwFunction{Shift}{Shift}
\DontPrintSemicolon
\newcommand{\Call}[1]{\text{\textsc{#1}}}

\SetCommentSty{algorithmcommentstyle}

\newtheorem{assumption}{Assumption} 
\newtheorem{lemma}{Lemma} 
\newtheorem{proposition}{Proposition} 

\newcommand{\X}{\ensuremath{\mathcal{X}}\xspace}
\newcommand{\U}{\ensuremath{\mathcal{U}}\xspace}
\newcommand{\T}{{\ensuremath{\mathcal{T}}}\xspace}

\usepackage{graphicx}
\usepackage{float}
\usepackage{dblfloatfix}
\usepackage{placeins}
\usepackage{afterpage}
\usepackage{wrapfig}
\usepackage{multirow}

\definecolor{changes}{RGB}{255, 10, 10}

%% file: sections/0.Abstract.tex
\begin{abstract}
\hyphenpenalty=10000
\exhyphenpenalty=10000
Sampling-based motion planners offer a practical and scalable approach to kinodynamic motion planning, notably for high-dimensional, underactuated, or non-holonomic systems. However, these planners are typically used offline, requiring execution to begin only after the trajectory has been computed. In addition, the planned trajectory may not be accurately tracked in the presence of motion uncertainty, leading to deviations from the nominal solution.
In this work, these limitations were addressed within a unified framework, \method, an asymptotically-optimal meta-planner framework that improves both path quality and tracking performance during execution.
In addition to the main execution thread, this framework comprises a replanning method that continuously explores the state space and refines the trajectory during execution, and an optimization process that refines future control inputs to reduce tracking error.
Together, these components enable \method to leverage asymptotically optimal planning online while improving execution accuracy under \textcolor{changes}{motion} uncertainty.
%
The proposed approach is evaluated in both simulation and real-world environments across multiple systems, demonstrating consistent improvements in trajectory quality, tracking accuracy, and overall performance compared with baseline\delete{method}s.
%
\textcolor{changes}{More information can be found on the project website: \url{https://elpislab.org/projects/aura}.}
\end{abstract}

%% file: sections/1.Introduction.tex
\section{Introduction}

\IEEEPARstart{S}{ampling-based} kinodynamic planners provide a general and theoretically sound framework for motion planning in systems with complex dynamics~\cite{sampling-based}. By relying solely on forward propagation, they avoid the need for steering functions, which are often intractable for arbitrary dynamical systems. Moreover, asymptotically optimal (AO) variants~\cite{asymptotically-optimal} offer \emph{probabilistic completeness} and \emph{asymptotic optimality}, ensuring that the probability of finding a feasible trajectory approaches one and that the cost converges almost surely to the global optimal as the number of samples increases.

Despite \delete{the strong theoretical}\textcolor{changes}{these} guarantees, two practical issues hinder the execution of trajectories computed by kinodynamic planners. Consider the kinodynamic planning problem of pushing an object in~\autoref{fig:intro}. First, the initial \delete{kinodynamic} trajectory is often \emph{suboptimal} due to limited planning time. Second, \emph{motion uncertainty}, arising from model mismatch and actuation errors, inevitably induces deviation from the initial trajectory during open-loop execution.
\begin{figure}[ht!]
    \centering
    \includegraphics[width=0.95\columnwidth]{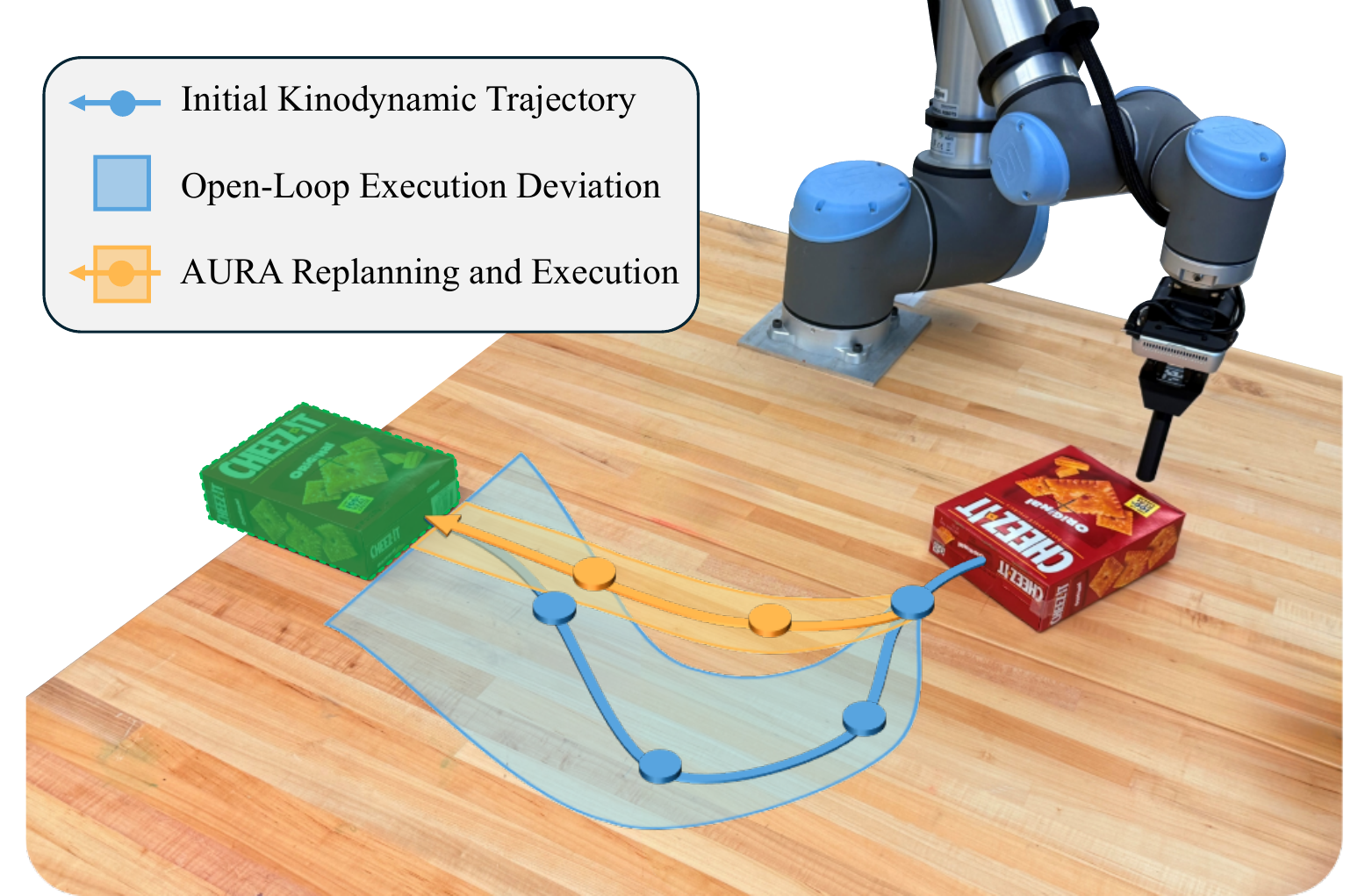}
    \vspace{-1ex}
    \caption{\textbf{Illustration of open-loop and \method execution for non-prehensile manipulation.} Under a limited time budget and with an approximated dynamics model, open-loop execution of the planned path yields a suboptimal trajectory and accumulates large execution deviation (\textcolor{cyan}{blue}). In contrast, \method performs online global replanning and local optimization, resulting in a better overall trajectory and smaller execution error (\textcolor{YellowOrange}{orange}).}
    \label{fig:intro}
    \vspace{-1ex}
\end{figure}
Online replanning strategies\mbox{\cite{ramp,rrt-x,kraft}} attempt to handle trajectory suboptimality by recomputing trajectories from the latest observed state. However, for kinodynamic AO planners, frequent replanning is computationally expensive\delete{and}\textcolor{changes}{,} may discard prior exploration progress\delete{. Furthermore}, \textcolor{changes}{and} without a steering function, reconnecting to the original search tree is not trivial~\cite{sst}.

Meanwhile, receding-horizon control~\cite{original-mpc, robust-control, dob} mitigates motion uncertainty through closed-loop tracking of a reference trajectory. However, the reference is typically fixed and not improved during execution. Moreover, these methods often rely on high-frequency state feedback~\cite{high-mpc}. In certain tasks—such as the non-prehensile manipulation scenario shown in~\autoref{fig:intro}—state observations may only be available \delete{intermittently}\textcolor{changes}{at low-frequency discrete action boundaries}. (e.g., after each push due to in-hand sensing).

In this work, we propose \method, \textbf{A}symptotically Optimal \textbf{U}ncertainty-Robust \textbf{R}eplanning \textbf{A}lgorithm that combines the global exploration and asymptotic guarantees of AO kinodynamic planners~\mbox{\cite{sst,ao-rrt}} with the local robustness of receding-horizon control. Rather than treating planning and execution independently, \method operates as a meta-layer on top of any forward-propagation-based AO planner, enabling: (i) continuous online refinement of trajectory quality during execution, and (ii) local control optimization to mitigate execution deviation. Concretely, the contributions of \method are:
\begin{itemize}
    \item A novel online replanning framework for kinodynamic AO planners that continuously improves trajectory quality \textcolor{changes}{online} without requiring a steering function.
    \item A GPU-accelerated local control optimization module to \textcolor{changes}{empirically} mitigate execution deviation up to 72\%, \delete{and a}\textcolor{changes}{together with an existence} proof that a recovering control always exists under mild assumptions.
    \item Comprehensive evaluation across multiple dynamics models in simulated and real-world tasks, including learned models, demonstrating up to 50\% reduction in total task time over receding-horizon and \delete{replanning}\textcolor{changes}{planning} baselines.
\end{itemize}

%% file: sections/2.RelatedWork.tex
\section{Related Work}
This section reviews\delete{the} prior work most relevant to \delete{\method}\textcolor{changes}{this method} from three perspectives: asymptotically\textcolor{changes}{-}optimal kinodynamic planning, online replanning, and robust control\delete{and optimization} under uncertainty. 
\delete{Our method}\textcolor{changes}{\method} lies at the\textcolor{changes}{ir} intersection\delete{of these areas by}\textcolor{changes}{,} combining\delete{ an} AO kinodynamic planner\textcolor{changes}{s} with global replanning and local optimization.

\subsection{Asymptotically-Optimal Kinodynamic Planners}
Sampling-based planners for kinodynamic systems must satisfy system dynamics and constraints~\cite{planning-book}. These methods have been successfully applied to a variety of systems, while avoiding the need for an explicit steering function\delete{~\cite{kpiece}}.
Among such methods, SST\textsuperscript{*} achieves asymptotic optimality through pruning and best-node selection~\cite{sst}.
More generally, the meta-algorithm \mbox{AO-X}~\cite{ao-rrt} lifts any kinodynamic planner, e.g., RRT~\cite{rrt} or EST~\cite{est}, into asymptotically optimal planners by augmenting the search in state-cost space. 

However, AO planners guarantee convergence to the optimal path only in probability as computation time increases~\cite{asymptotically-optimal}. 
Therefore, under limited time scenarios, the resulting solution may remain substantially suboptimal.
To address this challenge, \method exploits the extra time during execution, enabling the planner to further refine and improve the trajectory quality.

\subsection{Online Replanning}
AO planning is typically cast as an offline problem, since the planning terminates once a trajectory is selected for execution.
In contrast, real-time planning methods\delete{ are designed to} couple planning and execution, enabling online adaptation based on the current state during execution, but generally fail to offer optimality~\cite{2008-replanning}.
Early work in this direction\delete{ includes}\textcolor{changes}{,} RAMP \delete{'s extension to}\textcolor{changes}{for} non-holonomic\delete{ kinodynamic} systems~\mbox{\cite{ramp, ramp-nonholonomic}},\delete{which} leverages replanning during execution to adapt to environmental changes and action uncertainty.
Moreover, RRT\textsuperscript{X}~\cite{rrt-x} enables reconnecting to the original tree at runtime \delete{in the presence of}\textcolor{changes}{when} environment changes.
However, these methods require a steering function for local rewiring.

A recently proposed kinodynamic replanning framework, KRAFT~\cite{kraft}, avoids the need for the steering function. Nevertheless, it re-propagates the full tree after the state deviates from the nominal trajectory, which cannot guarantee\delete{ the} collision-free execution without access to the true dynamics model. In contrast, \method computes a recovery control that steers the system back to the collision-free planned trajectory without discarding the planning progress or \delete{ having access}\textcolor{changes}{requiring} to the true dynamics\delete{ of the} model.
%
%
%

\subsection{Control and Optimization under Uncertainty}
On the execution side, \delete{a broad class of control and optimization methods has}\textcolor{changes}{several control strategies have} been developed to improve tracking accuracy and robustness under \textcolor{changes}{motion} uncertainty\delete{. Representative examples include}\textcolor{changes}{, including} Model Predictive Control (MPC)~\cite{original-mpc}, \textcolor{changes}{and }robust control~\cite{robust-control}\delete{, and disturbance-observer-based control (DOB)~\cite{dob}}.
Among these, MPC and its variants can explicitly handle system dynamics and constraints~\cite{receding-horizon, complicated-dynamics-mpc}. Related extensions include differentiable MPC formulations that enable online parameter adaptation~\cite{diff-mpc}, and sampling-based stochastic control methods, such as MPPI~\cite{mppi}. 
Despite their strong execution performance, they are susceptible to local minima in long-horizon tasks.

A common solution is to follow a pre-planned reference trajectory~\cite{tube-mpc, navigation-mpc}, but these trajectories remain fixed throughout execution\delete{,} and are not refined.
%
\method combines AO kinodynamic planning with predictive optimization during execution\delete{,} to improve execution accuracy under \textcolor{changes}{motion }uncertainty while continuing to refine the overall trajectory quality.

%% file: sections/3.ProblemStatement.tex
\begin{figure*}[t]
    \centering
    \includegraphics[width=0.92\textwidth]{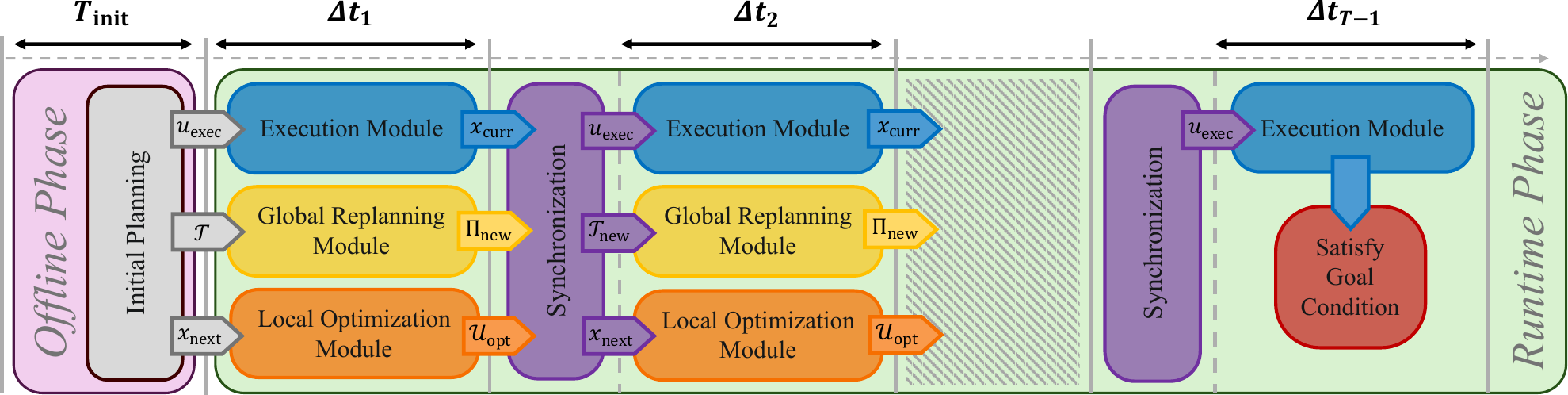}
    \caption{\textbf{Time frame of \method algorithm}, which includes \textit{offline} and \textit{runtime} phase. The \textit{offline} phase provides an initial guide for the next phase. At \textit{runtime}, three modules work concurrently for each $\Delta t_{\textcolor{changes}{i}}$ interval: \textcolor{MidnightBlue}{(a)} Execution module applies the next control and returns the resulting state, \textcolor{Dandelion}{(b)} global replanning module expands the existing planning tree, \textcolor{changes}{and }\textcolor{orange}{(c)} optimization module predicts the possible future controls based on the next state. \textcolor{Purple}{(d)} At the end of the cycle, all the data is synced to provide the inputs for the next cycle.}
    \label{fig:overview}
\vspace{-1ex}
\end{figure*}

\section{Problem Statement}
In this work, the kinodynamic motion planning problem under bounded execution uncertainty is considered.
Let $\X$ and $\U$ denote the state space and control space, respectively. The state space $\X$ is partitioned into two disjoint subsets, $\X = \X_{\text{obs}} \;\cup\; \X_{\text{free}}$, where $\X_{\text{obs}}$ represents the invalid state region and $\X_{\text{free}}$ is the valid state region.
The \textit{nominal dynamics model} $f$ describes the evolution of the system:
\begin{equation}
\label{eq:continues-dynamics}
x(T) = x(0) +\int_0^{T} f(x(t), u(t)) dt,
\end{equation}
where $x(t) \in \mathcal{X}$ and $u(t) \in \mathcal{U}$ denote the state and control input of the system at time $t$.
In practice, forward propagation is commonly performed over time intervals $\Delta t_{\textcolor{changes}{i}}$ using piecewise-constant controls~\cite{sampling-based}. Let $\Gamma(\cdot,\cdot)$ denote the \textit{nominal state transition function} over one such interval:
\begin{equation}
\label{eq:discrete-dynamics}
x_{i+1} = \Gamma(x_i, u_i) = x_i + \int_0^{\Delta t_{\textcolor{changes}{i}}} f(x(t), u_i) dt, \quad x(0) = x_i
\end{equation}
where $x_i$ is the start state and $u_i$ is the constant control applied \textcolor{changes}{for propagation duration $\Delta t_i \in [t_{\text{min}}, t_{\text{max}}]$}.

\textcolor{changes}{Let} \delete{A \textit{trajectory} $\pi$ is defined as} $\textcolor{changes}{\pi =(}x_0, u_0, \dots, x_{T-1}, u_{T-1}, x_T\textcolor{changes}{)}$ \textcolor{changes}{denote a trajectory}, where $\boldsymbol{x}\, \text{\delete{$:\{x_i\}_{i=0}^{T}$}}\, \textcolor{changes}{=(x_i)_{i=0}^{T}\in}\, \X_{\textcolor{changes}{\text{free}}}^{\textcolor{changes}{T+1}}$ is the state trajectory generated by applying the control sequence $\boldsymbol{u}\, \text{\delete{$:\ u_i\}_{i=0}^{T-1}$}}\, \textcolor{changes}{=(u_i)_{i=0}^{T-1}\in}\, \U^{\textcolor{changes}{T}}$ from $x_0$, satisfying~\autoref{eq:discrete-dynamics}. Let $J(\cdot)$ represent the \textit{cost function} that typically measures the quality of a trajectory, defined as:
\begin{equation}
\label{eq:cost_integral}
J(\pi) = \textstyle\sum\nolimits_{i=0}^{T-1} \ell(x_i, u_i),
\end{equation}
where $\ell: \mathcal{X} \times \mathcal{U} \rightarrow \mathbb{R}^+$ is the running cost (e.g., path length, time, or energy consumption).

During \textit{\delete{trajectory} execution}, the true system state may deviate from the nominal state $x_{i+1}$ due to model mismatch, unmodeled dynamics, and actuation errors. This \delete{execution} deviation is modeled as:
\begin{equation}
\label{eq:ground-truth-dynamics}
x^{\text{gt}}_{i+1} = \Gamma(x_i, u_i) + w_i,
\end{equation}
where $x^{\text{gt}}_{i+1}$ denotes the observed ground-truth system state after applying control $u_i$ \delete{over one interval of duration}\textcolor{changes}{for} $\Delta t_{\textcolor{changes}{i}}$, and $w_i$ is the \textit{bounded execution \textcolor{changes}{uncertainty}}, satisfying $\|w_i\| \le \delta$. Accordingly, the \textit{tracking error} along trajectory $\pi$ is defined~as:
\begin{equation}
\label{eq:tracking_error}
E(\pi) = \textstyle\sum\nolimits_{i=0}^{T} \|x_i - x_i^{\mathrm{gt}}\|,
\end{equation}
where $x_i$ and $x_i^{\mathrm{gt}}$ denote the nominal and observed states after execution, respectively.

The \textcolor{changes}{robust} kinodynamic planning problem \textcolor{changes}{under motion uncertainty} is formalized as follows:
Given an initial state $x_{\text{start}}$, a goal region $\X_{\text{goal}}$, and the nominal propagation model $\Gamma\textcolor{changes}{(\cdot,\cdot)}$, \method aims to determine a nominal trajectory $\pi$, such that:
\begin{equation}
\label{eq:kino_traj}
\begin{gathered}
x_{i+1}=\Gamma(x_i,u_i),\: i=0,\ldots,T-1, \\
x_0=x_{\text{start}},\: x_T\in\X_{\text{goal}}, \:
u_{0:T-1}\in\U^T, \: x_{0:T},x^{\mathrm{gt}}_{0:T}\in\X_{\text{free}}^{T+1}.
\end{gathered}
\end{equation}
while seeking to minimize the trajectory cost in~\autoref{eq:cost_integral} and the tracking error in~\autoref{eq:tracking_error}\delete{. Thereby, \method is designed to}\textcolor{changes}{, thereby} improve trajectory quality \delete{while maintaining}\textcolor{changes}{and} robustness to bounded motion uncertainty during execution.

%% file: sections/4.Methodology.tex
\section{\method Framework}
The main idea of \method is shown in \autoref{alg:aura} and \autoref{fig:overview}. 
Given an AO planner $\mathcal{P}$ and the nominal forward propagator $\Gamma$, \method aims to produce a trajectory that satisfies the goal condition while minimizing a cost function and execution error.
\textcolor{changes}{Let $\T$ denote the search tree maintained by $\mathcal{P}$, and \mbox{$\Pi = \{ \pi_1, \pi_2, ..., \pi_N \}$} denote the set of root-to-goal plans in $\T$.}
In the \textit{Offline Phase}, \delete{a set of initial trajectories \mbox{$\Pi = \{ \pi_1, \pi_2, ..., \pi_N \}$} is computed by the planner $\mathcal{P}$}\textcolor{changes}{$\mathcal{P}$ plans} \delete{within}\textcolor{changes}{for} planning time $T_{\text{init}}$, and \textcolor{changes}{selects} the best trajectory \textcolor{changes}{$\pi^\star \in \Pi$} \delete{is selected} to start execution
(\autoref{alg:aura}, \autoref{aura:plan_init} - \autoref{aura:get_next_state1}). 
Afterward, the \textit{Runtime Phase} begins, which consists of three parallel modules running together: \textbf{(a)} Execution Module, \textbf{(b)} Global Replanning Module, and \textbf{(c)} Local Optimization Module. Rather than operating sequentially, these modules run concurrently during trajectory execution. This procedure \delete{is repeated}\textcolor{changes}{repeats} iteratively until the system reaches the goal region.
%
\subsection{Execution Module}
This module, shown as blue blocks in~\autoref{fig:overview}, is responsible for \delete{executing}\textcolor{changes}{applying} the first control action $u_{\text{exec}}$ of the current best trajectory \textcolor{changes}{for its associated duration~$\Delta t_i$} and observing the resulting ground-truth state $x_{\text{curr}}^{\text{gt}}$\delete{The control is applied to the current state for a fixed duration $\Delta t$}, during which the system evolves based on~\autoref{eq:ground-truth-dynamics} (\autoref{alg:aura}, \autoref{aura:apply_control}).
In simulation, the system \delete{'s evolution is computed by}\textcolor{changes}{evolves based on} the physics engine after applying the control \textcolor{changes}{input}. In the real world, the consequent state results directly from the system's physical response to the executed control. In both cases, the resulting ground-truth state is then observed and passed to the synchronization module. In this work, the observation is assumed to be \textcolor{changes}{accurately} available \textcolor{changes}{at the end of each execution interval} without measurement noise.

{%

\newcommand{\auraAlgorithmFont}{%
  \fontsize{9.5pt}{10.8pt}\selectfont
}

\SetAlFnt{\auraAlgorithmFont}

\SetNlSty{auraAlgorithmFont}{}{}

\newcommand{\auraCommentStyle}[1]{%
  {\normalfont
   \ttfamily
   \mdseries
   \auraAlgorithmFont
   \color{wwpi}#1}%
}

\SetCommentSty{auraCommentStyle}

\SetKwComment{AuraLeftComment}{/*\ }{}

\SetKwComment{AuraRightMarker}{}{}

\newcommand{\AuraThread}[1]{%
  \AuraLeftComment*[l]{#1 */}%
}

\newcommand{\AuraThreadRef}[2]{%
  \makebox[0pt][l]{%
    \auraCommentStyle{/* #1 */}%
  }%
  \AuraRightMarker*[r]{%
    {\normalfont
     \auraAlgorithmFont
     \color{wwpi}%
     \hyperref[#2]{\(\triangleright\) Alg.~\ref*{#2}}}%
  }%
}

\SetKwInput{KwOut}{Result}

\begin{algorithm}[t]

\caption{AURA Framework}
\label{alg:aura}

\KwIn{AO planner $\mathcal{P}$; Transition function $\Gamma (\cdot,\cdot)$; \allowbreak
      Start~state $x_{\text{start}}$; \allowbreak
      Goal~region $\mathcal{X}_{\text{goal}}$; \allowbreak
      Initial~planning~time $T_{\text{init}}$; \allowbreak
      \delete{Control~duration~$\Delta t$;} \allowbreak
      \textcolor{changes}{Neighborhood~sampling~radius~$\sigma$}; \allowbreak}

\KwOut{\textcolor{changes}{Executed} control sequence driving
$x_{\text{start}}$ to $\mathcal{X}_{\text{goal}}$}

\BlankLine
$(\Pi, \mathcal{T}) \gets
\Call{Plan}(\mathcal{P}, \Gamma, x_{\text{start}},
\mathcal{X}_{\text{goal}}, \text{time}=T_{\text{init}})$\;
\nllabel{aura:plan_init}

$x_{\text{curr}}^{\text{gt}} \gets x_{\text{start}}$\;
\nllabel{aura:set_curr}

$\pi^\star \gets \Call{GetBestPlan}(\Pi)$\;
\nllabel{aura:get_best}

$u_{\text{exec}} \gets \Call{GetFirstControl}(\pi^\star)$\;
\nllabel{aura:first_control}

$x_{\text{next}} \gets \Call{GetNextPlannedState}(\pi^\star)$\;
\nllabel{aura:get_next_state1}

\BlankLine
\While{$x_{\text{curr}} \notin \mathcal{X}_{\text{goal}}$}{
\nllabel{aura:while_loop}

  \AuraThread{Execution Thread}

  $x_{\text{curr}}^{\text{gt}} \gets
  \Call{Execute}(x_{\text{curr}}^{\text{gt}},
  u_{\text{exec}}\text{\delete{, $\Delta t$}})$\;
  \nllabel{aura:apply_control}

  \AuraThreadRef{Replanning Thread}{alg:replanning}

  $(\Pi, \mathcal{T}) \gets
  \Call{Replanning}(\mathcal{P},
  \Gamma,
  \mathcal{T},
  x_{\text{next}},
  \mathcal{X}_{\text{goal}},
  \Delta t_{\textcolor{changes}{u_{{\text{exec}}}}})$\;
  \nllabel{aura:replanning}

  \AuraThreadRef{Optimization Thread}{alg:optimization}
  \label{alg1:optim}

  $(\mathcal{X}_{\text{child}}, \mathcal{U}_{\text{child}})
  \gets \Call{GetChildren}(x_{\text{next}}, \mathcal{T})$\;
  \nllabel{aura:get_children}

  $\X_\mathcal{B} \gets
  \Call{SampleNearby}(x_{\text{next}}, \sigma)$\;
  \nllabel{aura:sample_nearby}

  $\mathcal{U}_{\text{opt}} \gets
  \Call{Optimization}(\X_\mathcal{B},
  \mathcal{X}_{\text{child}},
  \mathcal{U}_{\text{child}}, \Gamma)$\;
  \nllabel{aura:optimization}

  \AuraThread{Synchronization}

  $\Call{AwaitThreads()}$\;
  \nllabel{aura:await_threads}

  $\Pi \gets
  \Call{RecalculateCost}(x_{\text{curr}}^{\text{gt}}, \Pi)$\;
  \nllabel{aura:recalc_cost}

  $\pi^\star \gets \Call{GetBestPlan}(\Pi)$\;
  \nllabel{aura:get_best2}

  $x_{\text{next}} \gets
  \Call{GetNextPlannedState}(\pi^\star)$\;
  \nllabel{aura:get_next_state2}

  $x^\star \gets
  \Call{Nearest}(x_{\text{curr}}^{\text{gt}},
  \X_\mathcal{B}\cup\{x_{\text{next}}\})$\;
  \nllabel{aura:nearest}

  $u_{\text{exec}} \gets
  \mathcal{U}_{\text{opt}}[x^\star, x_{\text{next}}]$\;
  \nllabel{aura:set_control}
}

\end{algorithm}
}

\subsection{Global Replanning Module}
\label{sec:replanning}
As noted earlier, the key property of the AO planner $\mathcal{P}$ is asymptotic convergence to the optimal trajectory as the number of samples\delete{ $N$} increases~\cite{asymptotically-optimal}.
%
%
This motivates resuming planning from the existing tree to find higher-quality trajectories~(\autoref{fig:replanning}). Two main operations occur in this module: \textit{Pruning} and \textit{Replanning}~(\autoref{aura:replanning} of~\autoref{alg:aura}, i.e.,~\autoref{alg:replanning}).

\textbf{Pruning:} After selecting \delete{the control to execute}\textcolor{changes}{$u_{\text{exec}}$}, the search tree needs to be updated to remain kinodynamically consistent.
As shown in~\autoref{fig:replanning}, during execution toward $x_1$, branches that are not descendants of $x_1$ (faded nodes) can no longer be used.
Finding a suitable control to connect $x_1$ to these nodes constitutes a \textit{2-Point Boundary Value Problem (2PBVP)} that requires a steering function to solve, to which this work does not assume access.
The next state \delete{on the current best solution}\textcolor{changes}{of $\pi^\star$}, $x_1$, is set as the new root (\autoref{alg:replanning}, \autoref{plan:set_start})\delete{. Then}\textcolor{changes}{, and} all nodes and branches that are not descendants of this new root are removed to construct the tree~$\T_{\text{pruned}}$ (\autoref{alg:replanning}, \autoref{plan:prune}).

\textbf{Replanning:} The planner then continues refining the trajectory, and runs for the remaining execution interval $T_{\text{replan}}$ with a new start state, starting with the existing tree $\T_{\text{pruned}}$. As $\mathcal{P}$ is an AO planner, there is a non-zero probability that a lower-cost trajectory can be found~\cite{asymptotically-optimal}, an example of which is illustrated by the highlighted golden trajectory in~\autoref{fig:replanning}.

\begin{figure}[t]
    \centering
    \includegraphics[width=0.95\columnwidth]{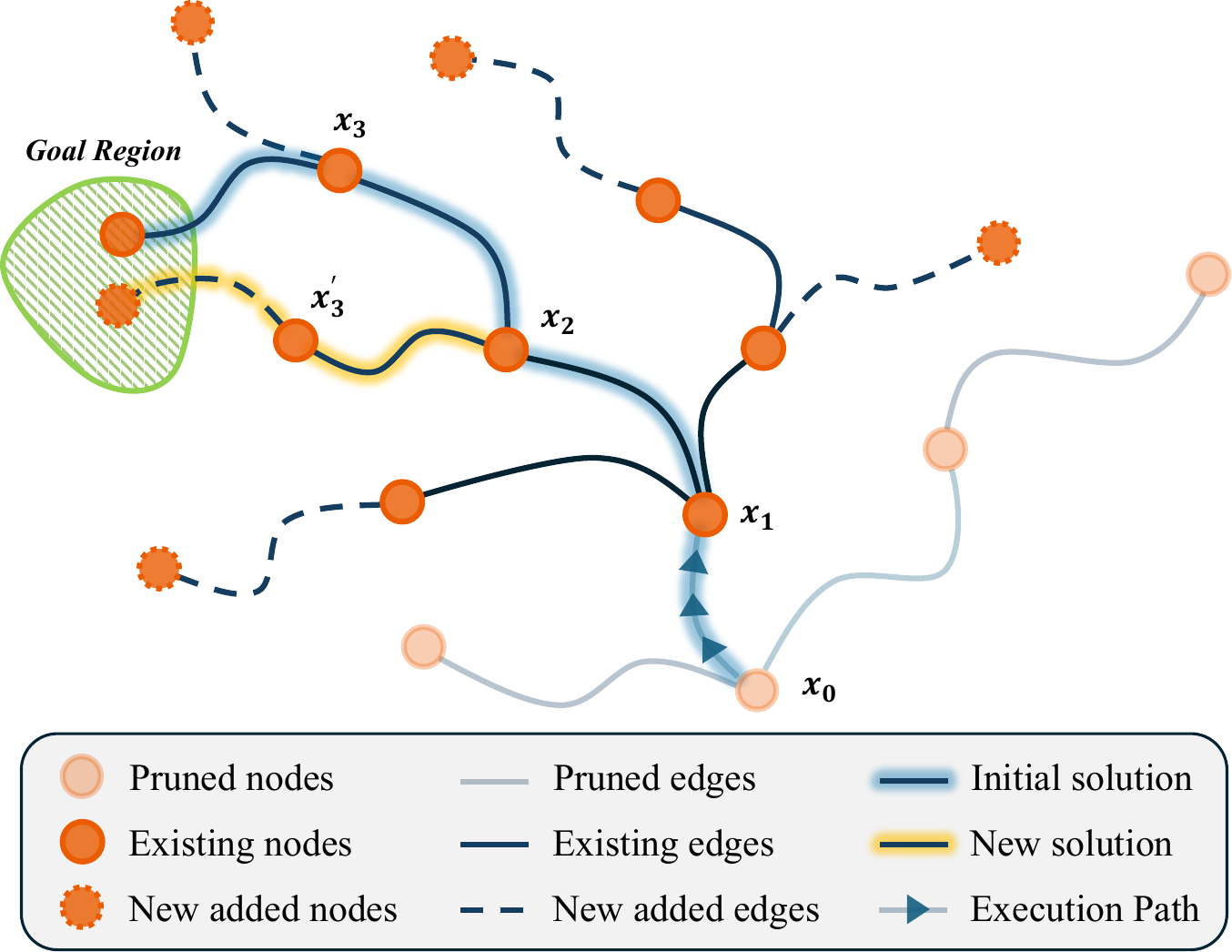}
    \vspace{-1ex}
    \caption{\textbf{\delete{An example of the tree adjustment}\textcolor{changes}{Tree update} during the runtime phase.} As the first segment of the solution is executed from $x_0$ to $x_1$, nodes and edges that are no longer reachable are pruned (faded), and new samples are added to the tree (dashed). If a\delete{new} lower-cost solution (\textcolor{Dandelion}{golden}) is found, it \delete{is chosen as}\textcolor{changes}{becomes} the \delete{current}\textcolor{changes}{new} best trajectory~$\pi^\star$.}
    \label{fig:replanning}
\end{figure}

{%

\newcommand{\replanningAlgorithmFont}{%
  \fontsize{9.5pt}{10.8pt}\selectfont
}

\SetAlFnt{\replanningAlgorithmFont}

\SetNlSty{replanningAlgorithmFont}{}{}

\begin{algorithm}[t]

\caption{Replanning}
\label{alg:replanning}

\KwIn{AO planner $\mathcal{P}$; Transition function $\Gamma (\cdot,\cdot)$; \allowbreak
      Tree~$\mathcal{T}_{\text{curr}}$;
      Next state $x_{\text{next}}$; \allowbreak
      Goal region $\mathcal{X}_{\text{goal}}$;
      Replanning time $T_{\text{replan}}$;}

\KwOut{Updated plans $\Pi_{\text{new}}$, Updated tree $\mathcal{T}_{\text{new}}$}

$x_{\text{start}} \gets x_{\text{next}}$\;
\nllabel{plan:set_start}

$\mathcal{T}_{\text{pruned}} \gets
\Call{PruneUnreachableNodes}(\mathcal{T}_{\text{curr}},
x_{\text{start}})$\;
\nllabel{plan:prune}

$(\Pi_{\text{new}}, \mathcal{T}_{\text{new}}) \gets$
\textsc{Plan}
$\!\big(\mathcal{P},x_{\text{start}},
\mathcal{X}_{\text{goal}}, \Gamma,
T_{\text{replan}}, \mathcal{T}_{\text{pruned}}\big)$\;
\nllabel{plan:plan_call}

\Return $(\Pi_{\text{new}}, \mathcal{T}_{\text{new}})$
\nllabel{plan:return}

\end{algorithm}
}
\vspace{-2ex}

\subsection{Local Optimization Module}
\label{sec:aura-opt}
This module employs a novel precomputation-based optimization strategy to reduce the tracking error defined in~\autoref{eq:tracking_error}. In contrast to MPC and other receding-horizon methods, which compute \delete{an optimized}\textcolor{changes}{the} control\delete{sequence} after observing the updated\delete{system} state, \delete{this method}\textcolor{changes}{\method} computes candidate recovery controls for the next cycle before the actual state is observed, during the current execution interval. The resulting optimized control set~$\U_{\text{opt}}$ serves as the candidate pool for selecting the execution control in the next cycle.

\begin{figure}[ht!]
  \centering
  \includegraphics[width=\columnwidth]{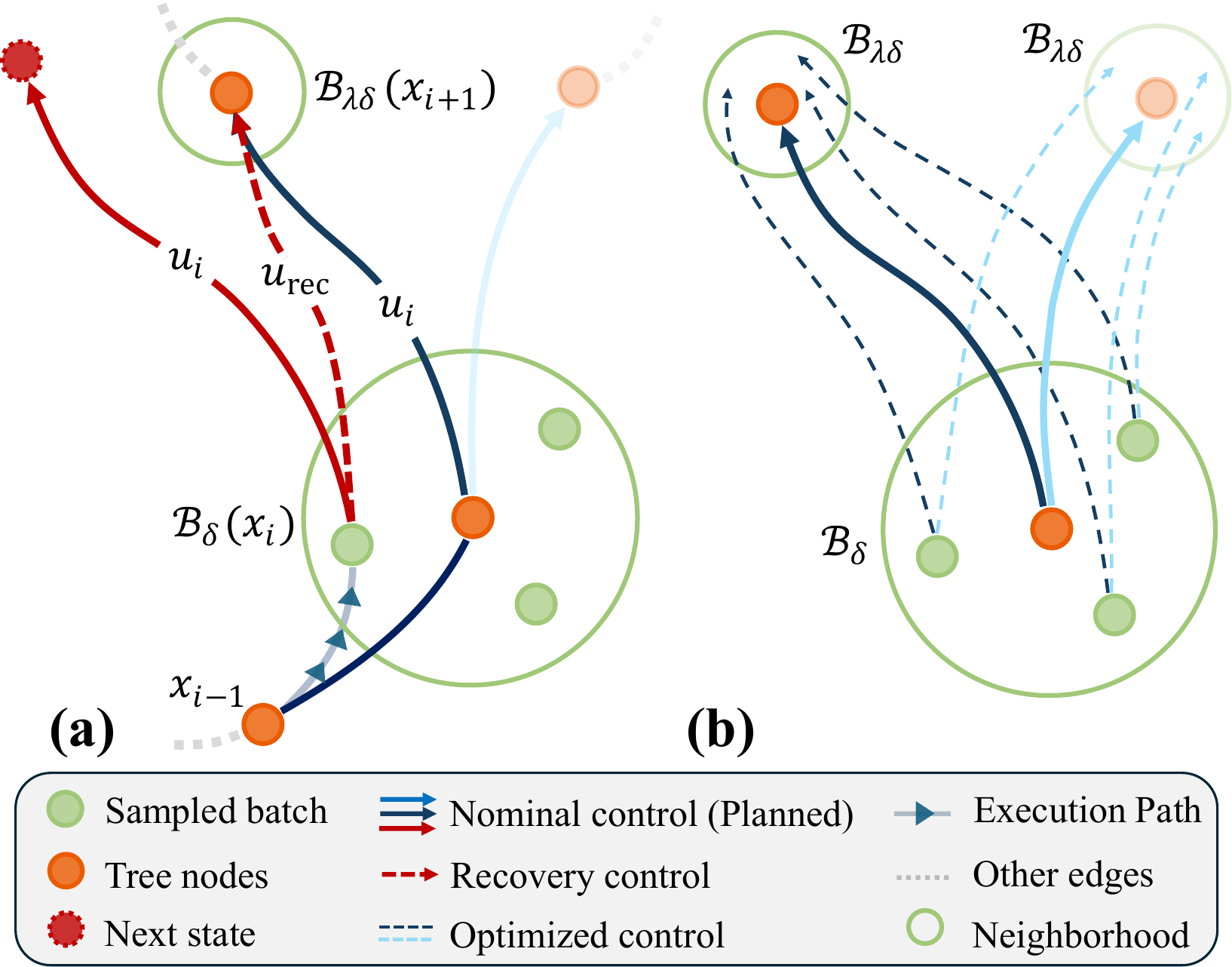}
  \vspace{-1ex}
  \caption{\textbf{Batched optimization process}. \textbf{(a)} This is a trajectory segment. When the executed state deviates from the desired state ($x_i$), applying the nominal control (red arrow) can further amplify the deviation. The optimization module seeks to approximate the optimal recovery control (dashed red arrow). \textbf{(b)} For each sample in the neighborhood of $x_i$ (\textcolor{YellowGreen}{green nodes}) and each child on the tree (\textcolor{orange}{orange nodes}), \method calculated an optimized control (dashed lines).}
  \label{fig:optimization}
  \vspace{-1ex}
\end{figure}

{%

\newcommand{\optimizationAlgorithmFont}{%
  \fontsize{9.5pt}{10.8pt}\selectfont
}

\SetAlFnt{\optimizationAlgorithmFont}

\SetNlSty{optimizationAlgorithmFont}{}{}

\begin{algorithm}[t]

\caption{Optimization}
\label{alg:optimization}

\KwIn{Sampled local nodes $\X_\mathcal{B}$;
      Children nodes
      $\mathcal{X}_{\text{\textcolor{changes}{child}}}$;
      Corresponding controls
      $\mathcal{U}_{\text{\textcolor{changes}{child}}}$;
      Transition $\Gamma (\cdot,\cdot)$;}

\KwOut{Optimized controls $\mathcal{U}_{\text{opt}}$}
\BlankLine

\mbox{$(X_{\text{start}}, X_{\text{child}}, U^0)
\gets
\Call{AssembleBatch}(
\X_\mathcal{B},
\mathcal{X}_{\text{\textcolor{changes}{child}}},
\mathcal{U}_{\text{\textcolor{changes}{child}}})$\;}
\nllabel{opt:assemble}

\For{$k \gets 0$ \KwTo $N$}{
\nllabel{opt:for_loop}

  $\mathcal{L} \gets
  \left\|
  \Gamma\!\left(X_{\text{start}},\,U^{k}\right)
  -X_{\text{children}}
  \right\|^2$\;
  \nllabel{opt:loss}

  $U^{k+1} \gets
  \textcolor{changes}{
    \Call{clip}_{\mathcal U}\!\left(
      \textcolor{black}{
        U^{k}
        -\alpha
        \cfrac{\partial\mathcal{L}}{\partial U^k}
      }
    \right)
  }$\;
  \nllabel{opt:update}
}

\Return $U^N$
\nllabel{opt:return}

\end{algorithm}
}

In the presence of motion uncertainty (\autoref{eq:ground-truth-dynamics}), executing the control will lead the system to deviate from the nominal state $x_i$ to the true state $x_i^{\text{gt}} \in \mathcal{B}_\delta(x_i)$, where $\mathcal{B}_\delta(x_i)$ denotes the neighborhood of $x_i$ defined as a ball of radius $\delta$.
As shown by the solid red line in \autoref{fig:optimization}~\textbf{(a)}, this deviation can accumulate along the trajectory.
\autoref{prop:recovery_segment} in \autoref{sec:theory} proves that there exists a recovery trajectory segment, generated by a control function $u_{\text{rec}}$, from the perturbed state back toward the nominal trajectory.
This motivates the use of numerical optimization to search for a local control that best approximates such a recovery by minimizing the terminal propagation error.

While the execution is underway, this module samples a batch of states $\X_\mathcal{B}$ in the neighborhood of the next successor state (\autoref{alg:aura}, \autoref{aura:get_children} - \autoref{aura:sample_nearby}) as candidates for possible \delete{rollout}\textcolor{changes}{execution} outcomes, illustrated as green states in~\autoref{fig:optimization}. For each sampled candidate, the next planned control provided by the planner (solid lines in \autoref{fig:optimization}~\textbf{(b)}) serves as an initial guess for an optimization process that adjusts control inputs to get closer to the next children (dashed line in~\autoref{fig:optimization}~\textbf{(b)}).


\autoref{alg:optimization} shows the optimization procedure\delete{, which is performed} over all pairs between the $b$ sampled states in\delete{the batch} $\X_\mathcal{B}$ and the $c$ child states in $\X_{\text{\textcolor{changes}{child}}}$. The optimizer computes one recovery control for each sampled state--child pair, producing a set of $b \times c$ optimized controls.
To enable efficient GPU-based parallel optimization, the sample--child inputs are vectorized into aligned matrices. Each sampled state is repeated $c$ times, once for each child state, while the child states and corresponding\delete{original} planned controls are repeated for each of the $b$ sampled states. This forms three aligned matrices with $b \times c$ entries: initial states $X_{\text{start}}$, target states $X_{\text{\textcolor{changes}{child}}}$, and \delete{original}\textcolor{changes}{planned} controls $U^0$ (\autoref{alg:optimization}, \autoref{opt:assemble}).

To numerically estimate $u_{\text{rec}}$ for each sample-child combination and minimize the error $\| \Gamma(X_{\text{start}}, U) - X_{\text{\textcolor{changes}{child}}} \|$, \textcolor{changes}{this module employs} gradient descent \delete{is employed}\textcolor{changes}{with step size $\alpha$} using the\delete{derivative of the} dynamics \textcolor{changes}{gradient} with respect to the control inputs (\autoref{alg:optimization}, \autoref{opt:for_loop} - \autoref{opt:update}). \textcolor{changes}{After each update, the controls are projected onto the admissible set $\U$.} This\delete{general} formulation makes \method compatible with any differentiable dynamics\delete{equation}, including analytical functions and learned\delete{neural network} models\textcolor{changes}{, while gradients for unknown dynamics can be approximated using batched finite differences.}

\subsection{Synchronization}
At this stage of the process, the execution thread has advanced the system to a new state $x_{\text{curr}}$, the global replanning process has expanded the tree from the new root, and the optimization process has generated a $\U_{\text{opt}}$ set to mitigate deviation.

After the execution, any small deviations can alter the relative costs of previously computed trajectories. Therefore, at the end of each cycle, the costs of the stored candidate solutions are recalculated (\autoref{alg:aura}, \autoref{aura:recalc_cost}) based on the current state.
%
%
Accordingly, the next state along the current best trajectory $\pi^\star$ is chosen as the next target state (\autoref{alg:aura}, \autoref{aura:get_best2} - \autoref{aura:get_next_state2}).
Additionally, the\delete{sampled} state in $\X_\mathcal{B}$ that is \textcolor{changes}{nearest} to\delete{ the current system state} $x_{\text{curr}}$ is selected as the best estimate (\autoref{alg:aura}, \autoref{aura:nearest}).
Based on this estimate and the target state, the best optimized control is selected from\delete{the precomputed set} $\mathcal{U}_{\text{opt}}$ to be executed in the next cycle (\autoref{alg:aura}, \autoref{aura:set_control}). \textcolor{changes}{If the target state is a newly added child, its recovery control is optimized on demand, since the local optimization module had not included the new child in the precomputed batch.}

%% file: sections/5.Theoretical.tex
\section{Recovery Control Existence Proof}
\label{sec:theory}

This section proves the existence of a recovery control that takes a perturbed state back toward the nominal successor along a planned trajectory, as shown in~\autoref{fig:optimization} \textbf{(a)}. 

%
\begin{assumption}[]
\label{as:lipschitz}
The system dynamics $\Gamma$ is Lipschitz continuous in both the state and control variables. Formally, $\exists K^\Gamma_x, K^\Gamma_u \ge 0$ such that, $\forall x, x' \in \X$, $\forall u, u' \in \U$:
\begin{equation}
\begin{gathered}
\label{eq:lip_f}
\|\Gamma(x, u) - \Gamma(x', u)\| \le K^\Gamma_x\|x - x'\|, \\
\|\Gamma(x, u) - \Gamma(x, u')\| \le K^\Gamma_u\|u - u'\|.
\end{gathered}
\end{equation}
\end{assumption}

\begin{assumption}[]
\label{as:stla}
The system satisfies Chow's condition~\cite{sst}.
\end{assumption}
Chow's condition implies \emph{small-time local accessibility (STLA)}, meaning that the reachable set from a state contains an open neighborhood around that state. 

Such assumptions are standard in the analysis of kinodynamic sampling-based planners. Under these assumptions, we borrow the following result from~\cite[Lemma 6, Definition 7]{sst}.

\begin{lemma}[]
\label{lem:dynamic_clearance}
Let $\boldsymbol{x}$ be a \textcolor{changes}{feasible} state trajectory for a system satisfying~\autoref{as:lipschitz} and~\autoref{as:stla}. Then there exists a positive value $\delta_0$, termed dynamic clearance, such that for every $\eta \in (0, \eta_0]$, every $x'_0 \in \mathcal{B}_\delta(x_0)$, and every $x'_T \in \mathcal{B}_\delta(x_T)$, there exists a dynamically feasible state trajectory $\hat{\boldsymbol{x}}$ satisfying $\hat{x}_0 = x'_0$ and $\hat{x}_T = x'_T$. The trajectory is called $\delta$-robust if both its obstacle clearance $\epsilon$, i.e., minimum distance from obstacles over all trajectory states, and its dynamic clearance $\eta_0$ are greater than $\delta$.
\end{lemma}

With~\autoref{lem:dynamic_clearance}, we can obtain the following proposition.

\begin{proposition}[Existence of a Recovery Segment]
\label{prop:recovery_segment}
Let $\overline{x_i x_{i+1}}$ be a segment from $x_i$ to $x_{i+1}$ along a $\delta_0$-robust nominal state trajectory.
Then for any perturbed state \(x_i' \in \mathcal{B}_\delta(x_i)\) with \(\delta \le \delta_0\), there exists a dynamically feasible state trajectory segment $\overline{x_i' x_{i+1}}$ from \(x_i'\) to \(x_{i+1}\).
\end{proposition}

\begin{proof}
Consider a nominal $\delta_0$-robust trajectory segment \(\overline{x_i x_{i+1}}\).
By \autoref{lem:dynamic_clearance}, for any \(\delta \in (0,\delta_0]\), any initial state \(\hat{x}_0 \in \mathcal{B}_\delta(x_i)\), and any target state \(\hat{x}_T \in \mathcal{B}_\delta(x_{i+1})\), there exists a dynamically feasible trajectory segment connecting \(\hat{x}_0\) to \(\hat{x}_T\).
\delete{Now choose}\textcolor{changes}{Setting} \(\hat{x}_0 = x_i'\)\delete{, where \mbox{$x_i' \in \mathcal{B}_\delta(x_i)$} by assumption, and choose} \textcolor{changes}{and} \(\hat{x}_T = x_{i+1}\)\delete{.}\textcolor{changes}{,} since \(x_{i+1} \in \mathcal{B}_\delta(x_{i+1})\) for \(\delta > 0\), there exists a dynamically feasible trajectory\delete{segment} from \(x_i'\) to \(x_{i+1}\).
\end{proof}

\delete{This result}\textcolor{changes}{\autoref{prop:recovery_segment}} guarantees the existence of a local recovery trajectory segment, together with an admissible recovery control function $u_{\text{rec}}$\delete{ that generates it}.
\delete{Our}\textcolor{changes}{The proposed} optimization module does not claim to recover $u_{\text{rec}}$ exactly from the nominal initialization. Instead, it is a practical local search\delete{procedure} that attempts to approximate it with $u_{\text{opt}}$ by minimizing the terminal propagation error $\|\Gamma(x'_i, u_{\text{opt}}) - x_{i + 1}\|^2$.
Therefore, the theory motivates the optimization objective, while the optimization module (\autoref{sec:aura-opt}) serves as a numerical approximation mechanism.

Additionally, \autoref{prop:approx_recovery} shows that exact recovery is not required\delete{.}\textcolor{changes}{, and provides a conditional sufficient criterion. If the local optimization produces a control whose predicted successor is sufficiently close to the nominal target, and the remaining execution error is\delete{correspondingly} bounded,\delete{then} the executed successor remains inside the original $\delta$-neighborhood.}
If the execution errors were routinely on the order of $\delta$, then a planned $\delta$-robust trajectory would itself be unreliable, since the system could easily leave the clearance region and collide with obstacles.

\begin{figure*}[ht]
    \centering
    \includegraphics[width=\textwidth]{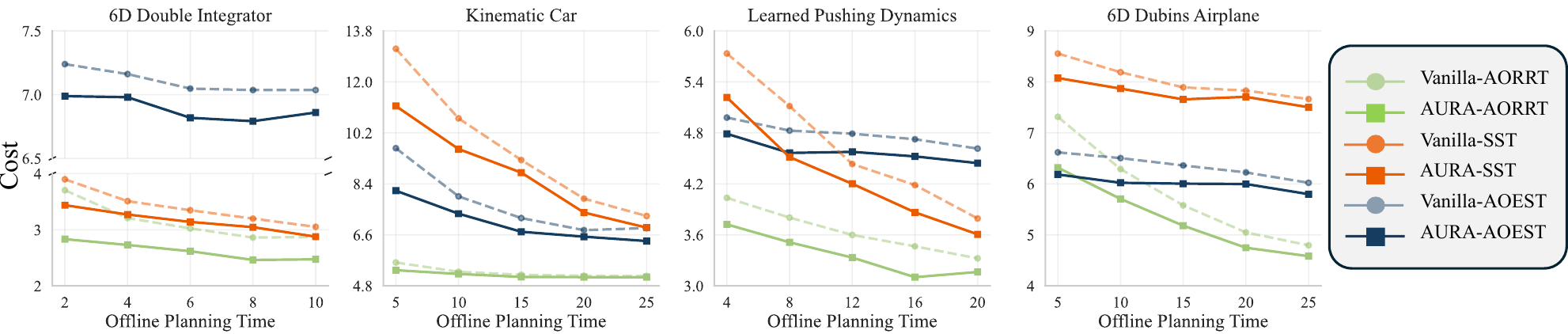}
    \vspace{-2ex}
    \caption{\textbf{Trajectory quality comparison.} The final cost \delete{function} value of the solution path is represented for three planners across \delete{three systems}\textcolor{changes}{four dynamics} \textcolor{changes}{averaged over 100 trials}. When integrated with \method,\delete{any} AO planner\textcolor{changes}{s} can further improve trajectory quality \textcolor{changes}{during execution} in the runtime phase, given any initial planning time.}
    \label{fig:cost-comparison}
\vspace{-2ex}
\end{figure*}

\begin{proposition}[Approximate Recovery Suffices]
\label{prop:approx_recovery}
Let \(x_i' \in \mathcal{B}_\delta(x_i)\), and suppose the local optimization module returns a control \(u_{\mathrm{opt}}\) such that $\Gamma(x_i', u_{\mathrm{opt}}) \in \mathcal{B}_{\lambda\delta}(x_{i+1})$ for some \(\lambda \in (0,1)\).
If the execution error\delete{$w$ of this optimized segment} is bounded by $(1-\lambda)\delta$, that is,
$\|x'_{i+1} - \Gamma(x_i', u_{\mathrm{opt}})\| \le (1-\lambda)\delta,$
where $x'_{i+1}$ denotes the executed successor state, then we have 
$ 
x'_{i+1} \in \mathcal{B}_\delta(x_{i+1}).
$ 
\end{proposition}

\begin{proof}
Applying the triangle inequality gives
\begin{equation}
\begin{aligned}
\|x'_{i+1}-x_{i+1}\|
&\le (1-\lambda)\delta+\|\Gamma(x_i',u_{\mathrm{opt}})-x_{i+1}\|\\
&\le (1-\lambda)\delta+\lambda\delta=\delta.
\end{aligned}
\end{equation}
Hence $x'_{i+1} \in \mathcal{B}_\delta(x_{i+1})$.
\end{proof}

%% file: sections/6.Experiments.tex
\section{Experiments and Results}
This section presents a detailed experimental evaluation of \method on multiple systems with different dynamics in both simulation and real-world environments. The results demonstrate the effect of different modules on the overall performance.
All the experiments were run on a workstation with an Ultra7-265KF CPU, an NVIDIA RTX4070 Ti Super GPU, and 32GB of RAM.
%
To demonstrate the generalizability of the proposed method, systems with different nominal dynamics and state spaces were investigated, including:
\begin{itemize}
    \item \textbf{6D Double Integrator:} A single-point acceleration control model used for forward propagation, with state space \mbox{$x = [p_x, p_y, p_z, \dot{p}_x, \dot{p}_y, \dot{p}_z]^\top \in \mathcal{X} \subseteq \R{6}$} where $p$ denotes position, and control space $u = [\ddot{p}_x, \ddot{p}_y, \ddot{p}_z]^\top \in \mathcal{U} \subseteq \R{3}$. 
    \item \textbf{Kinematic Car:} A planar \delete{car}\textcolor{changes}{reduced-order kinematic bicycle} model with state $x = [p_x, p_y, \theta]^\top \in \mathcal{X} \subseteq \mathrm{SE}(2)$ and control $u = [v, \phi]^\top \in \mathcal{U} \subseteq \R{2}$, where $\theta$, $v$, and $\phi$ denote yaw, \textcolor{changes}{longitudinal} velocity, and steering angle, respectively. The dynamics follow the model in~\cite{car-kinematics}.
    \item \textbf{Learned Non-prehensile Pushing:} A learned dynamics of contact-rich planar pushing following the model in~\cite{activepusher}, with state $x = [p_x, p_y, \theta]^\top \in \mathcal{X} \subseteq \mathrm{SE}(2)$ and control $u = [\textcolor{changes}{p_c}, u_d]^\top \in \mathcal{U} \subseteq \R{\textcolor{changes}{2}}$, where \textcolor{changes}{$p_c$ continuously parameterizes an object-relative contact point along the object boundary}, and $u_d$ specifies the pushing distance.

    \item \textcolor{changes}{\textbf{Nonlinear 6D Dubins Airplane:} A nonlinear 3D flight model with state $x =[p_x,p_y,p_z,\psi,\gamma,v]^\top \in \X \subseteq \R{6}$ and control $u = [a,\omega_\psi,\omega_\gamma]^\top \in \U \subseteq \R{3}$, where $\psi$, $\gamma$, and $v$ denote yaw, pitch, and speed, and $a$, $\omega_\psi$, and $\omega_\gamma$ denote acceleration, yaw-rate, and pitch-rate, respectively \cite{dubins-airplane}.}
\end{itemize}

The robustness of \method was evaluated across environments with different forms of execution uncertainty $w$ (\autoref{eq:ground-truth-dynamics}). 
 \begin{itemize}
     \item \textbf{Gaussian Noise (GN):} Execution uncertainty is modeled as bounded additive Gaussian noise applied to the nominal motion of each system.
     \item \textbf{Mujoco Simulation:} Uncertainty arises from the mismatch between the nominal dynamics and the MuJoCo simulation dynamics~\cite{mujoco}, evaluated using a MuSHR car model~\cite{mushr} \textcolor{changes}{with a simplified reduced-order model} (\textcolor{changes}{shown in} \autoref{fig:environments}~(a)) and a UR10 manipulation task (\autoref{fig:environments} (b)).
     \item \textbf{Real-World Robot:} Hardware inaccuracies and unmodeled dynamics introduce uncertainty in a tabletop manipulation task (\autoref{fig:environments} (c)).
 \end{itemize}

For each system and environment, three AO planners implemented in the \textit{Open Motion Planning Library (OMPL)}~\cite{ompl} were used: SST~\cite{sst}, AO-RRT, and AO-EST~\cite{ao-rrt}.

\begin{figure}[h!]
  \centering
  \vspace{-2ex}
  \includegraphics[width=\linewidth]{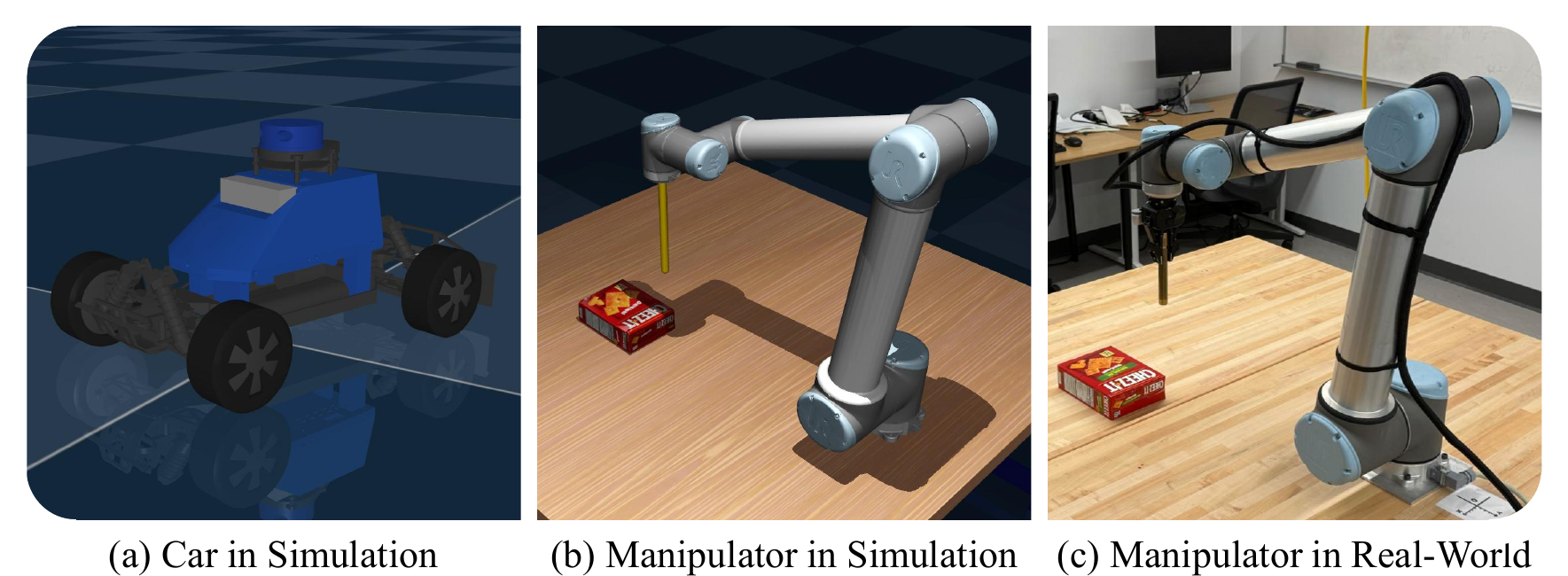}
  \vspace{-1ex}
  \caption{\textbf{Evaluation environments.} (a) MuJoCo MuSHR car, (b) MuJoCo UR10 non-prehensile task, and (c) Real-world UR10 non-prehensile task.}
  \label{fig:environments}
\end{figure}




\begin{figure*}[ht]
    \centering
    \includegraphics[width=\textwidth]{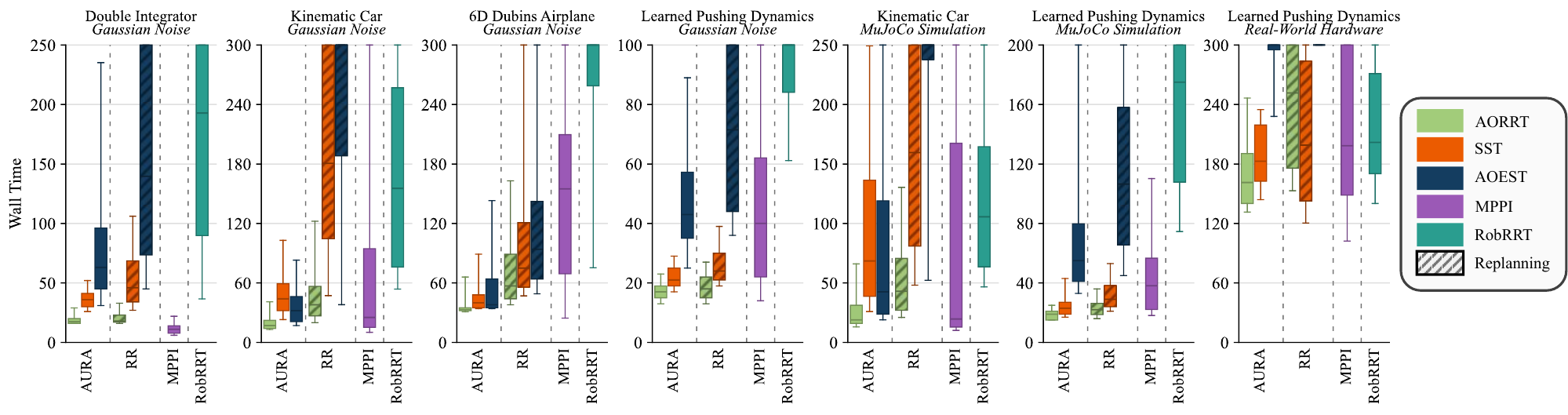}
    \vspace{-4ex}
    \caption{\textbf{Task time comparison.} The distribution of overall task completion time (in seconds) over multiple trials for \method, Restart Replanning, \delete{and}MPPI\textcolor{changes}{, and Robust-RRT} is shown \textcolor{changes}{over 100 trials}. \textcolor{changes}{The offline planning times were 10s for the double integrator, kinematic car, and Dubins airplane, and 5s for the pushing system.} Lower values indicate faster task completion and better performance.\delete{\method consistently achieves lower wall time and reduced variability compared to other baselines, with pronounced gains in complex and long-horizon scenarios.}}
    \vspace{-1ex}
    \label{fig:wall-time-comparison}
\end{figure*}

\subsection{Runtime Trajectory Quality Improvement}
\label{sec:cost_comparison}
As discussed in~\autoref{sec:replanning}, \method\delete{planners} utilize\textcolor{changes}{s} each execution interval $\Delta t_{\textcolor{changes}{i}}$ to continue the global search from the retained planning tree, \textcolor{changes}{and} progressively improv\delete{ing}\textcolor{changes}{es} the cost of the candidate trajectory\delete{until the system reaches the goal}.
This experiment evaluates the attained improvement in \delete{trajectory}\textcolor{changes}{this} cost \textcolor{changes}{value} $J(\pi)$ under a noise-free execution setting ($\|w\|=0$) \textcolor{changes}{on a matched initial planning budget $T_{\text{init}}$}, thereby isolating the effect of runtime global \delete{refinement}\textcolor{changes}{replanning} on solution quality from the effects of tracking error.
The baseline is the \delete{initial}trajectory generated by the corresponding vanilla AO planner \delete{during the offline phase}, as shown in~\autoref{fig:cost-comparison}. \delete{Results are averaged over 50 trials for each system, planner, and initial planning time.}\textcolor{changes}{As the execution deviation is zero, no recovery is required for this experiment.}

For all the systems \textcolor{changes}{and planners}, \method consistently produced lower-cost trajectories\delete{(vertical axis)} than \textcolor{changes}{the corresponding} vanilla planners, given the same offline planning time.
\delete{Notably, for the kinematic car, 5 seconds of offline planning with AURA could achieve a trajectory quality comparable to that obtained from 30 seconds of offline planning.}\textcolor{changes}{The benefit is more pronounced for shorter planning budgets and generally decreases as the offline planning time increases, since longer planning already returns lower-cost trajectories.}
A similar trend is observed \delete{across all systems, including}\textcolor{changes}{for} the learned dynamics model, where finding a solution is more challenging, \textcolor{changes}{and the nonlinear 6D Dubins airplane, where 3D dynamics are strongly coupled}.
This experiment therefore demonstrates \method~’s ability to reduce initial computation while continuing trajectory refinement during execution\delete{. These outcomes are}\textcolor{changes}{, which is} consistent with the theoretical properties of asymptotic optimality~\cite{asymptotically-optimal}.

\begin{table}[h]
\centering
\caption{Mean Step-wise Tracking Error}
\vspace{-1ex}
\label{tab:tracking_error}
\setlength{\tabcolsep}{3pt}
\renewcommand{\arraystretch}{1.5}
\footnotesize
\resizebox{\columnwidth}{!}{%
\begin{tabular}{llccc}
\hline
\textbf{System} & \textbf{Environment} & \textbf{AURA} & \textcolor{changes}{\textbf{MPPI}} & \textbf{Open-loop Execution} \\
\hline
Double Integrator & GN & 0.00175 & \textcolor{changes}{0.00194} & 0.00368 \\
Kinematic Car & GN & 0.03913 & \textcolor{changes}{0.03146} & 0.13798 \\
Pushing Dynamics & GN & 0.02903 & \textcolor{changes}{0.02671} & 0.05856 \\
\textcolor{changes}{Dubins Airplane} & \textcolor{changes}{GN} & \textcolor{changes}{0.06404} & \textcolor{changes}{0.06390} & \textcolor{changes}{0.14991}\\
Kinematic Car & MuJoCo & 0.04477 & \textcolor{changes}{0.03941} & 0.15385 \\
Pushing Dynamics & MuJoCo & 0.02015 & \textcolor{changes}{0.01942} & 0.03621 \\
Pushing Dynamics & Real-World & 0.08313 & \textcolor{changes}{0.07426} & 0.18560 \\
\hline
\end{tabular}
}
\end{table}

\subsection{Tracking Robustness}
To examine how local optimization affects execution under uncertainty, step-wise tracking error was measured using~\autoref{eq:tracking_error}.
This evaluation includes execution uncertainty, \delete{and the system is tasked with following an arbitrary reference trajectory generated by the nominal dynamics using 10 consecutive randomly sampled control inputs.}\textcolor{changes}{where for each trial, an arbitrary reference trajectory is generated by propagating 10 randomly sampled piecewise-constant controls through the nominal dynamics (\autoref{eq:discrete-dynamics}), with each control applied for one execution interval.} For real-world trials, this is reduced to 5 pushes. Results are averaged over \delete{5}\textcolor{changes}{100} trials \textcolor{changes}{for simulation and 20 for real-world}.
The comparison is made against open-loop execution, in which the nominal controls are applied \delete{directly}\textcolor{changes}{without feedback}, \textcolor{changes}{and MPPI, which observes the current state, optimizes a finite-horizon control sequence online, and executes its first control}. In \method, the best optimized control is selected at each step from $\mathcal{U}_{\text{opt}}$\delete{based on the current observed state and the next target state,} as described in~\autoref{sec:aura-opt}.

As reported in~\autoref{tab:tracking_error}, \delete{selecting the optimized control}\textcolor{changes}{\method} substantially decreases tracking error \textcolor{changes}{relative to open-loop execution} across all systems, \delete{yielding a reduction of up to \textbf{72\%} in simulation and \textbf{53\%} in the real world}\textcolor{changes}{and remains comparable to MPPI. Unlike MPPI, which performs optimization after observing the state, AURA precomputes candidate recovery controls during the execution and selects the best one at the synchronization point.}

\subsection{End-to-End Task Time Efficiency}
To evaluate the overall task performance, wall time was measured as the total elapsed time required for each system to reach the goal, under \delete{action}\textcolor{changes}{motion} uncertainty. This includes both the initial offline planning time and the total execution time.
%
The \delete{comparison}\textcolor{changes}{baselines} includes Restart Replanning (RR)\delete{and}\textcolor{changes}{,} Model Predictive Path Integral (MPPI)~\cite{mppi} \textcolor{changes}{and Robust-RRT (\mbox{RobRRT})~\cite{robust-rrt}}. Following common practice in replanning under uncertainty~\cite{planning-from-scratch}, \delete{restart replanning is triggered whenever the step-wise tracking error exceeds a threshold. This approach discards prior planning progress and computes a new trajectory}\textcolor{changes}{RR checks the step-wise deviation from its nominal trajectory. When this deviation exceeds a threshold, it discards the existing search tree and replans from the current observed state to the goal}. MPPI \delete{is included as a}\textcolor{changes}{provides} receding-horizon \textcolor{changes}{control} optimization\delete{baseline}. \textcolor{changes}{RobRRT incorporates execution uncertainty directly into planning using sampling-based reachable-set approximations~\cite{rand-up}.}

\delete{Results in}~\autoref{fig:wall-time-comparison} demonstrate that \method \delete{consistently}\textcolor{changes}{reliability} reduces wall time \delete{across systems and environments} by utilizing global replanning and local optimization modules. Although \delete{the} RR \delete{approach} uses the latest ground-truth state for planning, each \delete{trigger}\textcolor{changes}{restart} requires a costly trajectory recomputation. \delete{In contrast, \method progressively lowers trajectory cost while reducing tracking error, limiting the need for restart replanning. The benefit becomes most apparent in}\textcolor{changes}{In} more complex settings, such as MuJoCo simulation and real-world \delete{pushing tasks}\textcolor{changes}{setting},\delete{where} RR exhibits higher variance and longer execution~times. 

While MPPI \delete{achieves competitive performance for} \textcolor{changes}{outperforms \method on} the \delete{kinematic car}\textcolor{changes}{double integrator, indicating that finite-horizon optimization is sufficient for simple dynamics}, its wall time increases substantially in more complex systems. 
This is because MPPI optimizes over a limited horizon without explicit awareness of the global state space or long-term constraints. Therefore, the controller aligns locally feasible states and converges to local minima.
\textcolor{changes}{On the other hand, Robust-RRT have greater planning overhead with multiple forward propagations for each tree expansion, and it does not provide an asymptotic-optimality guarantee.}
In contrast, \delete{the long-horizon planner in \method enables}\textcolor{changes}{\method retains} global exploration and trajectory refinement, supporting more efficient task completion in complex systems and environments.

\subsection{Hyperparameter Analysis}
\begin{figure}[t!]
    \centering
    \includegraphics[width=\columnwidth]{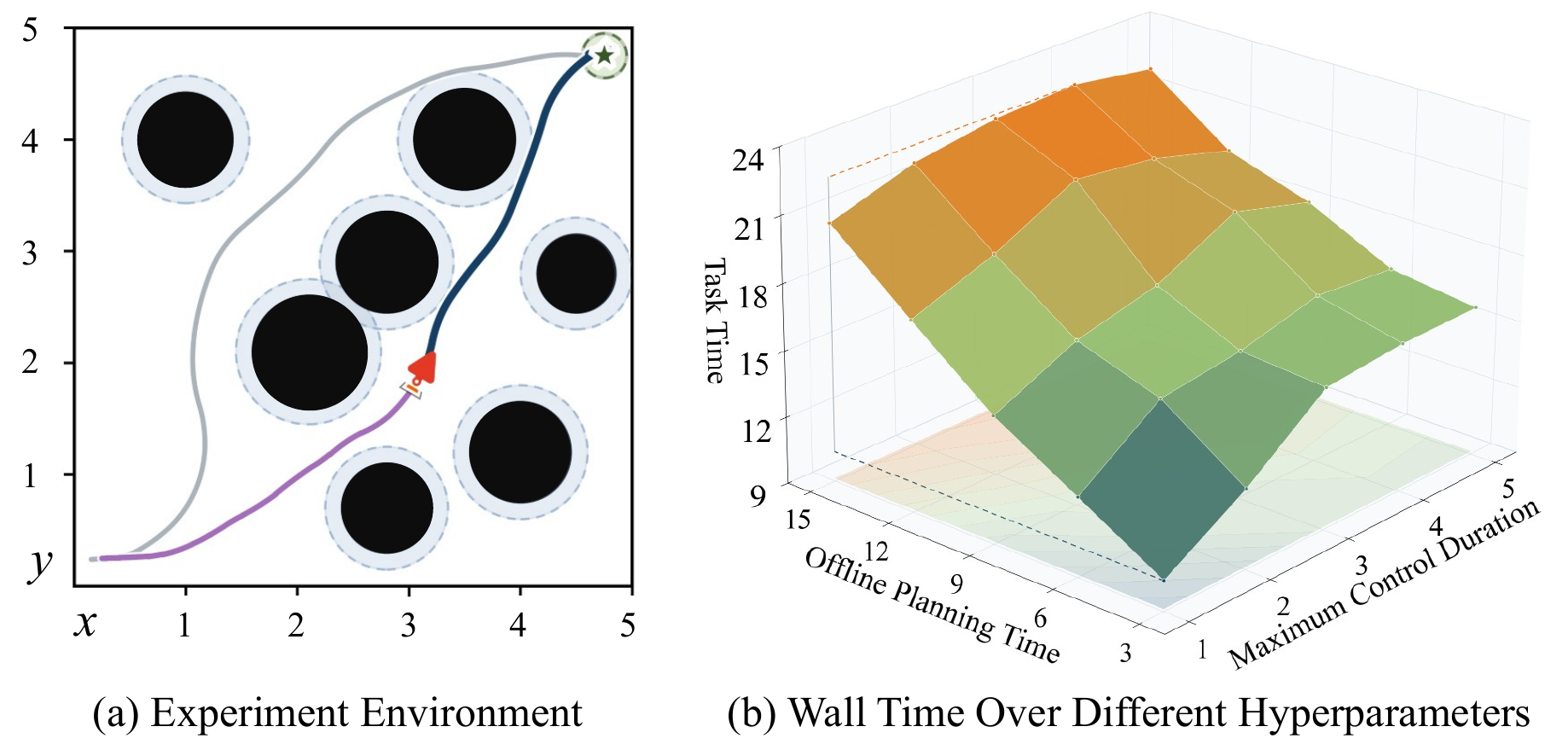}
    \vspace{-2ex}
    \caption{\textbf{Hyperparameters Analysis.} (a) \textit{XY} plane with obstacles (black) and their inflated boundaries (dashed lines). The \textcolor{changes}{red} simulated car is tasked with reaching the goal region (\textcolor{OliveGreen}{Green}). The \textcolor{gray}{gray} trajectory is the offline planning output, the \textcolor{Purple}{purple} segment shows the executed trajectory, and the \textcolor{Blue}{blue} line shows the best current trajectory computed by \method. (b) Effect of offline planning times and \textcolor{changes}{maximum} control durations on the overall wall time\delete{(z-axis)}.}
    \vspace{-1ex}
    \label{fig:hyperparameters-study}
\end{figure}
This section analyzes the sensitivity of \method to two critical hyperparameters: \textit{Offline Planning Time}~$T_{\text{init}}$ and \delete{\textit{Control~Duration}~$\Delta t$}\textcolor{changes}{\textit{Maximum Control Duration}~$t_{\max}$}, which set the initial computation budget and \delete{runtime interval frequency}\textcolor{changes}{upper bound on the sampled control duration}, respectively.
This experiment was conducted using the kinematic car with AO-RRT in the environment shown in~\autoref{fig:hyperparameters-study}~(a), where the \delete{wall time was measured over 10 trials}\textcolor{changes}{average wall time was measured over 50 trials} under different hyperparameter settings. \textcolor{changes}{For each planning edge, the control duration is sampled as $\Delta t_i\in[t_{\min},t_{\max}]$.} In this environment, the optimal solution passes through the narrow passage on the right, while a longer trajectory through the wider region on the left is more likely to be found during early planning.

As shown in~\autoref{fig:hyperparameters-study}~(b), larger offline planning time generally increases wall time. \delete{This implies that more initial planning time does not necessarily improve runtime efficiency. Using \method, the initial solution only needs to be feasible to begin execution, as the trajectory is continuously refined during the runtime phase. This allows the system to start with minimal offline computation while improving solution quality during~execution.}\textcolor{changes}{Since the reported wall time includes $T_{\text{init}}$, additional offline computation improves end-to-end efficiency only when the resulting trajectory improvement compensates for the added planning delay. This supports beginning execution once a feasible trajectory is available and continuing trajectory refinement during runtime.}

The \delete{control duration, however, has more influence on wall time. Shorter control durations enable more frequent replanning, allowing the global replanning module to update the trajectory more often. However, excessively short durations reduce the computation time available for local optimization, limiting the uncertainty compensation. Thus, the best performance is obtained at an intermediate control duration (which is 1 second for this system) that balances replanning frequency with sufficient optimization time.}
\textcolor{changes}{maximum control duration also influences wall time by changing the range of durations available to the planner and, consequently, the frequency of execution and replanning cycles. Across the tested settings, smaller $t_{\max}$ values generally result in lower wall time, although the trend is not strictly monotonic across offline planning budgets. The lowest average wall time, approximately $10s$, is obtained at $T_{\text{init}}=3s$ and $t_{\max}=1s$.}
%

Overall, this analysis indicates that \method's \delete{wall time grows with offline planning time with the same trend across different control durations, once a feasible initial trajectory is available. Beyond that point, any additional offline computation does not translate into improved performance. Instead, performance is primarily determined by the runtime frequency, which provides frequent replanning and enough time for optimization.}
\textcolor{changes}{wall time is generally better with short offline planning budgets. Since $t_{\max}=1$~s is the smallest value evaluated, it represents the best tested setting rather than a general optimum; the preferred maximum control duration depends on the system.}

%% file: sections/7.Conclusion.tex
\section[Limitation and Future Work]{\protect\delete{Conclusion}\textcolor{changes}{Limitation} and Future Work}
This paper presents \method, a meta-planning framework for robust motion planning under execution uncertainty. \method\delete{is built on top of asymptotically optimal sampling-based planners and} augments \delete{them}\textcolor{changes}{kinodynamic AO planners} with two complementary online modules: \textbf{(i)} a \textit{Global Replanning Module} that \delete{continuously refines}\textcolor{changes}{continue refining} the\delete{reference} trajectory during execution, and \textbf{(ii)} a \textit{Local Optimization Module} that \delete{recovers from}\textcolor{changes}{precomputes recovery controls for} possible \textcolor{changes}{execution} outcomes\delete{ by relying on forward dynamics propagation}. By combining \delete{asymptotic path improvement with predictive forward-dynamics}\textcolor{changes}{online global exploration and local recovery} optimization, \method improves trajectory quality\delete{during execution} while mitigating tracking errors \delete{under motion uncertainties}\textcolor{changes}{without requiring extensive offline planning}. Experimental results across \delete{multiple system}\textcolor{changes}{analytical and learned} dynamics\delete{and state spaces} demonstrate the generality and practical \delete{advantages}\textcolor{changes}{effectiveness} of this \delete{framework}\textcolor{changes}{method} \textcolor{changes}{under motion uncertainty}. 

Despite the advantages of the proposed method, some limitations remain. The local optimization module \delete{may exhibit varying recovery} performance \delete{depending}\textcolor{changes}{depends} on the sampled batch\delete{set} and nominal dynamics \delete{complexity}\textcolor{changes}{accuracy}. \textcolor{changes}{Consequently, tracking performance may degrade with poor optimizer convergence, or disturbances beyond the motion-uncertainty bound $\delta$. Moreover, the recovery guarantee assumes that a nominal trajectory with clearance greater than $\delta$ exists. When the environment contains passages narrower than the uncertainty tube, such a trajectory may not exist with this method. The current formulation also assumes static obstacles and noise-free observation.}\delete{In addition, the hyperparameter studies indicate that performance depends on selecting an appropriate control duration, which must be adapted based on the system's dynamics.}
\delete{As f}\textcolor{changes}{F}uture work, \textcolor{changes}{will investigate parallelizing planning computations for both offline and runtime phases}\delete{the computations}\delete{ in both the offline phase and the global replanning module could be parallelized}~\cite{kino-pax}, accelerating the overall pipeline and \delete{enabling the method to} better support systems \delete{that require}\textcolor{changes}{requiring} shorter control durations.